\documentclass[11pt]{article}
\usepackage{graphicx}
\newcount\Comments
\usepackage{bbm}
\usepackage{amsmath}
\usepackage{amssymb}
\usepackage{algorithm}
\usepackage[framemethod=TikZ]{mdframed}
\usepackage{mathtools}
\usepackage{amsthm}
\usepackage{natbib}
\usepackage{letltxmacro}
\usepackage{authblk}

\PassOptionsToPackage{algo2e,ruled, linesnumbered}{algorithm2e}
\RequirePackage{algorithm2e}

\LetLtxMacro\amsproof\proof
\LetLtxMacro\amsendproof\endproof

\usepackage{thmtools}
\usepackage{xcolor}
\AtBeginDocument{%
  \LetLtxMacro\proof\amsproof
  \LetLtxMacro\endproof\amsendproof
}

\newtheorem{theorem}{Theorem}
\newtheorem{proposition}[theorem]{Proposition}
\newtheorem{definition}{Definition}
\newtheorem{lemma}[theorem]{Lemma}
\newtheorem{corollary}[theorem]{Corollary}

\newtheorem*{theorem*}{Theorem}
\newtheorem*{remark*}{Remark}

\usepackage{xpatch}
\xpatchcmd{\proof}{\itshape}{\normalfont\proofnamefont}{}{}

\usepackage[most]{tcolorbox}

\usepackage[letterpaper,top=2cm,bottom=2cm,left=3cm,right=3cm,marginparwidth=1.75cm]{geometry}

\usepackage[colorlinks=true, allcolors=blue]{hyperref}

\DeclarePairedDelimiter\floor{\lfloor}{\rfloor}

\usepackage{color}
\definecolor{darkgreen}{rgb}{0,0.5,0}
\definecolor{purple}{rgb}{1,0,1}
\newcommand{\kibitz}[2]{\ifnum\Comments=1\textcolor{#1}{#2}\fi}

\newcommand{\Hcal}{\mathcal{H}}

\newcommand{\Tcal}{\mathcal{T}}

\newcommand{\Xcal}{\mathcal{X}}

\DeclareMathOperator*{\argmin}{arg\,min}

\newcommand{\indicator}{\mathbbm{1}}

\newcommand*\samethanks[1][\value{footnote}]{\footnotemark[#1]}

\title{Apple Tasting: Combinatorial Dimensions and Minimax Rates}
\author{Vinod Raman\thanks{Authors contributed equally}, Unique Subedi\samethanks , Ananth Raman, and Ambuj Tewari}
\date{}

\begin{document}

\maketitle

\begin{abstract}
     In online binary classification under \emph{apple tasting} feedback, the learner only observes the true label if it predicts ``1". First studied by \cite{helmbold2000apple},  we revisit this classical partial-feedback setting and study online learnability from a combinatorial perspective. We show that the Littlestone dimension continues to provide a tight quantitative characterization of apple tasting in the agnostic setting, closing an open question posed by \cite{helmbold2000apple}. In addition, we give a new combinatorial parameter, called the Effective width, that tightly quantifies the minimax expected mistakes in the realizable setting.  As a corollary, we use the Effective width to establish a \emph{trichotomy} of the minimax expected number of mistakes in the realizable setting. In particular, we show that in the realizable setting, the expected number of mistakes of any learner, under apple tasting feedback, can be $\Theta(1), \Theta(\sqrt{T})$, or $\Theta(T)$. This is in contrast to the full-information realizable setting where only $\Theta(1)$ and $\Theta(T)$ are possible. 
\end{abstract}

\section{Introduction}

In the standard online binary classification setting, a learner plays a repeated game against an adversary. In each round, the adversary picks a labeled example $(x, y) \in \mathcal{X} \times \{0,1\}$ and reveals the unlabeled example $x$ to the learner. The learner observes $x$ and then makes a prediction $\hat{y} \in \{0,1\}$. Finally, the adversary reveals the true label $y$ and the learner suffers the loss $\mathbbm{1}\{\hat{y} \neq y\}$ \citep{Littlestone1987LearningQW}. In many situations, receiving feedback after every prediction may not be realistic. For example, in spam filtering, emails that are classified as spam are often not verified by the user. Accordingly, the learner only receives feedback when an email is classified as ``not spam." In recidivism prediction, a person whose is predicted to re-commit a crime may not be released. Accordingly, we will not know whether this person would have re-committed a crime had they been released. Situations like these are known formally as ``apple tasting" \citep{helmbold2000apple}. In the generic model, a learner observes a sequence of apples, some of which may be rotten. For each apple, the learner must decide whether to discard or taste the apple. The learner suffers a loss if they discard a good apple or if they taste a rotten apple. Crucially,  when the learner discards an apple, they do not receive any feedback on whether the apple was rotten or not. 

Binary online classification under apple tasting feedback was first studied by \cite{helmbold2000apple} in the realizable setting.  Here, they give a simple and generic conversion of a deterministic online learner in the full-information setting into a randomized online learner in the apple tasting setting. In particular, they show that if $M_{+}$ and $M_{-}$ are upper bounds on the number of false positive and false negative mistakes of the deterministic online learner respectively, then the expected number of mistakes made by their conversion, under apple tasting feedback, is at most $M_{+} + 2\sqrt{T M_{-}}$. Along with these upper bounds, they provide lower bounds on the expected number of mistakes for randomized apple tasting learners in terms of the number of false positive and false negative mistakes made by any deterministic online learner in the full-information setting. That is, if there exists $M_{+}, M_{-} \in \mathbbm{N}$ such that every deterministic online learner in the full-information setting makes either at least $M_{+}$ false positive mistakes \emph{or} $M_{-}$ false negative mistakes, then every randomized online learner makes at least $\frac{1}{2}\min\left\{\frac{1}{2}\sqrt{T M_{-}}, M_{+}\right\}$ expected number of mistakes under apple tasting feedback. Finally, as an open question, they ask whether their results can be extended to the harder agnostic setting where the true labels can be noisy. 

While \cite{helmbold2000apple} establish bounds on the minimax expected number of mistakes in the realizable setting, their bounds are in terms of the existence of an algorithm with certain properties. This is in contrast to much of online learning theory, where minimax regret is often quantified in terms of combinatorial dimensions that capture the complexity of the hypothesis class \citep{Littlestone1987LearningQW, ben2009agnostic, DanielyERMprinciple, rakhlin2015online}. Accordingly, we revisit apple tasting and study online learnability from a combinatorial perspective. In particular, we are interested in identifying combinatorial dimensions that tightly quantify the minimax regret for apple tasting in both the realizable and agnostic settings. To that end, our main contributions are:
\begin{itemize}
    \item[(1)] We close the open question posed by \cite{helmbold2000apple} by showing that the minimax expected regret in the agnostic setting, under apple tasting feedback, is at most $3\sqrt{\operatorname{L}(\mathcal{H}) T \log(T) }$ and at least  $\sqrt{\frac{\operatorname{L}(\mathcal{H}) T}{8} }$, where $\operatorname{L}(\mathcal{H})$ is the Littlestone dimension of $\mathcal{H}$. 
    \item[(2)] On the other hand, we show that the Littlestone dimension alone does not give a tight quantitative characterization in the realizable setting. Instead, we show that the minimax expected number of mistakes in the realizable setting, under apple tasting feedback is 
    $$\Theta\Bigl(\max \Bigl\{\sqrt{(\operatorname{W}(\mathcal{H})-1)T}, 1\Bigl\}\Bigl),$$
    where $\operatorname{W}(\mathcal{H})$ is the \emph{Effective width} of $\mathcal{H}$, a new combinatorial parameter that accounts for the asymmetric nature of the feedback. 
    
    \item[(3)] Using the bound above, we establish the following trichotomy on the minimax rates in the realizable setting: (i) $\Theta(1)$ when $\operatorname{W}(\mathcal{H}) = 1$, (ii) $\Theta(\sqrt{T})$ when $1< \operatorname{W}(\mathcal{H}) < \infty$, and (iii) $\Theta(T)$ when $\operatorname{W}(\mathcal{H}) = \infty$. 
\end{itemize}
To prove (1), we extend the EXP3.G algorithm from \cite{alon2015online} to binary prediction with expert advice. Then, we use the standard technique from \cite{ben2009agnostic} to construct an agnostic learner using a realizable, mistake-bound learner in the full-information setting. To prove the upper bound in (2), we define a new combinatorial parameter, called the Effective width, and use it to construct a deterministic online learner in the realizable, \emph{full-information} feedback setting with constraints on the number of false positive and false negative mistakes. We then use this online learner and a conversion technique from \cite{helmbold2000apple} to construct a randomized online learner in the realizable, \emph{apple tasting} feedback setting with the stated guarantee in (2). For the lower bound in (2), we consider a new combinatorial object called an \emph{apple tree} and use it to explicitly construct a hard, realizable stream for any randomized, apple tasting learner. This is in contrast to \cite{helmbold2000apple}, who prove lower bounds on the minimax expected number of mistakes by converting randomized apple tasting learners into deterministic full-information feedback learners.

\subsection{Related Works}

Apple tasting is usually presented as an example of a more general partial feedback setting called \emph{partial monitoring} games, where the player’s feedback is specified by a feedback matrix \citep{cesa2006prediction, bartok2014partial}. Of particular interest is the work by \cite{bartok2012role}, who gives a beautiful result (Theorem 2) characterizing the minimax rates in different partial monitoring games (including apple tasting). However, this is done for a slightly different setting where there is no hypothesis class $\mathcal{H}$, but just a \emph{finite} set of actions the learner can play. The goal here is to compete with the best fixed \emph{action} in hindsight. In contrast, in our setting, there is a hypothesis class, often \emph{infinite} in size, and the goal is compete against the best fixed \emph{hypothesis} in hindsight. Related to partial monitoring games is sequential prediction with \emph{graph feedback}, for which apple tasting feedback is also special case \citep{alon2015online}.  In this model, a learner plays a repeated game against an adversary. In each round, the learner selects one of $K$ actions but observes the losses for a subset of the actions determined by a combinatorial structure called a \emph{feedback graph}. \cite{alon2015online} classify feedback graphs into three types and establish a trichotomy on the rates of the minimax regret based on the type of graph. In this paper, we extend the online learner presented in \cite{alon2015online} to the setting of binary prediction with expert advice to establish the minimax regret of apple tasting in the agnostic setting.


In a parallel direction, there has been an explosion of work using combinatorial dimension to give tight quantitative characterizations of online learnability. For example, \cite{Littlestone1987LearningQW} proposed the Littlestone dimension and showed that it exactly characterizes the optimal mistake bound of deterministic learners for online binary classification in the full-information, realizable setting.  Later, \cite{ben2009agnostic} show that the Littlestone dimension also provides a tight quantitative characterization of the optimal expected regret in the full-information, agnostic setting. Later, \cite{DanielyERMprinciple} define a \emph{multiclass} extension of the Littlestone dimension and show that it provides a tight quantitative characterization of realizable and agnostic mutliclass online learnability under full-information feedback when the label space is finite. In their same work, \cite{DanielyERMprinciple} define the Bandit Littlestone dimension and show that it exactly characterizes the optimal mistake bound of deterministic learners in the realizable setting under partial feedback setting known as bandit feedback. \cite{daniely2013price} and \cite{raman24multiclass} later show that the Bandit Littlestone dimension also characterizes agnostic bandit online learnability. Beyond binary and multiclass classification, combinatorial dimensions have been used to characterize online learnability for regression \citep{rakhlin2015online}, list classification \citep{moran2023list}, ranking \citep{raman2023online}, and general supervised online learning models \citep{raman2023combinatorial}.



\section{Preliminaries}

\subsection{Notation}
Let $\Xcal$ denote the instance space and $\mathcal{H} \subseteq \{0, 1\}^{\mathcal{X}}$ denote a binary hypothesis class. Given an instance $x \in \mathcal{X}$, and any collection of hypothesis $V \subseteq \{0, 1\}^{\mathcal{X}}$, we let $V(x) := \{h(x): h \in V\}$ denote the projection of $V$ onto $x$. As usual, $[N]$ is used to denote $\{1, 2, \ldots, N\}$. 

\subsection{Online Classification and Apple Tasting}

In the standard binary online classification setting with full-information feedback a learner $\mathcal{A}$ plays a repeated game against an adversary over $T$ rounds. In each round $t \in [T]$, the adversary picks a labeled instance $(x_t, y_t) \in \mathcal{X} \times \{0, 1\}$ and reveals $x_t$ to the learner. The learner makes a (possibly randomized) prediction $\mathcal{A}(x_t) \in \{0, 1\}$. Finally, the adversary reveals the true label $y_t$ and the learner suffers the 0-1 loss $\mathbbm{1}\{\mathcal{A}(x_t) \neq y_t\}.$ Given a hypothesis class $\mathcal{H} \subseteq \{0, 1\}^{\mathcal{X}}$, the goal of the learner is to output predictions such that its \emph{expected regret}
$$\operatorname{R}_{\mathcal{A}}(T, \mathcal{H}) := \sup_{(x_1, y_1), ..., (x_T, y_T)} \left(\mathbbm{E}\left[\sum_{t=1}^T \mathbbm{1}\{\mathcal{A}(x_t) \neq y_t \}\right] - \inf_{h \in \mathcal{H}} \sum_{t=1}^T \mathbbm{1}\{h(x_t) \neq y_t\}\right)$$
\noindent is small, where the expectation is only over the randomness of the learner. A hypothesis class $\mathcal{H}$ is said to be online learnable under full-information feedback, if there exists an (potentially randomized) online learning algorithm $\mathcal{A}$ such that $\operatorname{R}_{\mathcal{A}}(T, \mathcal{H}) = o(T)$ while $\mathcal{A}$ receives the true label $y_t$ at the end of each round. If it is guaranteed that the learner always observes a sequence of examples labeled by some hypothesis $h \in \mathcal{H}$, then we say we are in the \emph{realizable} setting and the goal of the learner is to minimize its \emph{expected cumulative mistakes},
$$\operatorname{M}_{\mathcal{A}}(T, \mathcal{H}) := \sup_{h \in \Hcal}\sup_{x_1, ..., x_T}\mathbbm{E}\left[\sum_{t=1}^T \mathbbm{1}\{\mathcal{A}(x_t) \neq h(x_t)\} \right],$$
 where again the expectation is taken only with respect to the randomness of the learner. In the apple tasting feedback model, the adversary still picks a labeled instance $(x_t, y_t) \in \mathcal{X} \times \{0, 1\}$ and reveals $x_t$ to the learner. However, the learner only gets to observe the true label $y_t$ if they predict $\hat{y}_t = 1$. Analogous to the full-information setting, a hypothesis class $\mathcal{H} \subseteq \{0, 1\}^{\Xcal}$ is online learnable under apple tasting feedback, if there exists an online learning algorithm whose expected regret, \emph{under apple tasting feedback}, on any sequence of labeled instances is $o(T)$.

\begin{definition} [Agnostic Online Learnability under Apple Tasting Feedback]
\label{def:appltasting}
\noindent A hypothesis class $\Hcal$ is online learnable under apple tasting feedback, if there exists an algorithm $\mathcal{A}$ such that 
$\operatorname{R}_{\mathcal{A}}(T, \mathcal{H}) = o(T)$  while $\mathcal{A}$ only receives feedback when predicting $1$. 
\end{definition}

As in the full-information setting, if it is guaranteed that the sequence of examples is labeled by some hypothesis $h \in \mathcal{H}$, then we say we are in the \emph{realizable} setting and an analogous definition of learnability under apple tasting feedback follows.

\begin{definition} [Realizable Online Learnability under Apple Tasting Feedback]
\label{def:appltasting}
\noindent A hypothesis class $\Hcal$ is online learnable under apple tasting feedback in the \emph{realizable setting}, if there exists an algorithm $\mathcal{A}$ such that $\operatorname{M}_{\mathcal{A}}(T, \mathcal{H}) = o(T)$  while $\mathcal{A}$ only receives feedback when predicting $1$. 
\end{definition}

\subsection{Trees and Combinatorial Dimensions}
In online learning, combinatorial dimensions are defined in terms of \emph{trees}, a basic unit that captures temporal dependence. A binary tree $\mathcal{T}$ of depth $d$ is \emph{complete} if it admits the following recursive structure. A depth one complete binary tree is a single root node with left and right outgoing edges. A complete binary tree $\mathcal{T}$ of depth $d$ has a root node whose left and right subtrees are each complete binary trees of depth $d-1$. Given a complete binary tree $\mathcal{T}$, we can label its internal nodes and edges by elements of $\mathcal{X}$ and $\{0, 1\}$ respectively to get a \emph{Littlestone tree}.

\begin{definition}[Littlestone tree]
\noindent A Littlestone tree of depth $d$ is a \emph{complete} binary tree of depth $d$ where the internal nodes are labeled by instances of $\mathcal{X}$ and the left and right outgoing edges from each internal node are labeled by $0$ and $1$ respectively.
\end{definition}

Given a Littlestone tree $\mathcal{T}$ of depth $d$, a root-to-leaf path down $\mathcal{T}$ is a bitstring $\sigma \in \{0, 1\}^d$ indicating whether to go left ($\sigma_i = 0$) or to go right ($\sigma_i = 1$) at each depth $i \in [d]$. A path $\sigma \in \{0, 1\}^d$ down $\mathcal{T}$ gives a sequence of labeled instances $\{(x_i, \sigma_i)\}_{i = 1}^{d}$, where $x_i$ is the instance labeling the internal node following the prefix $(\sigma_1, ..., \sigma_{i-1})$ down the tree. A hypothesis $h_{\sigma} \in \mathcal{H}$ shatters a path $\sigma \in \{0, 1\}^d$, if for every $i \in [d]$, we have  $h_{\sigma}(x_i) = \sigma_i$. In other words, $h_{\sigma}$ is consistent with the labeled examples when following $\sigma$. A Littlestone tree $\mathcal{T}$ is shattered by $\mathcal{H}$ if for every root-to-leaf path $\sigma$ down $\mathcal{T}$, there exists a hypothesis $h_{\sigma} \in \mathcal{H}$ that shatters it. Using this notion of shattering, we define the Littlestone dimension of a hypothesis class. 

\begin{definition}[Littlestone dimension]
\noindent The Littlestone dimension of $\mathcal{H}$, denoted $\operatorname{L}(\mathcal{H})$, is the largest $d \in \mathbbm{N}$ such that there exists a Littlestone tree $\mathcal{T}$ of depth $d$ shattered by $\mathcal{H}$. If there exists shattered Littlestone trees $\mathcal{T}$ of arbitrary depth, then we say that $\operatorname{L}(\mathcal{H}) = \infty$. 
\end{definition}

 Remarkably, the Littlestone dimension gives a tight quantitative characterization of realizable learnability under full-information feedback. In particular, \cite{Littlestone1987LearningQW} gives a generic deterministic algorithm, termed the Standard Optimal Algorithm (SOA), and shows that it makes at most $\operatorname{L}(\mathcal{H})$ number of mistakes on any realizable stream. Moreover, they showed that for every deterministic learner, there exists a realizable stream that can force at least  $\operatorname{L}(\mathcal{H})$ mistakes, proving that the Ldim exactly quantifies the mistake bound for deterministic realizable learnability under full-information feedback.

 Under apple tasting feedback, one can use the lower and upper bounds derived by \cite{helmbold2000apple} to deduce that the Ldim also provides a \emph{qualitative} characterization of realizable learnability. However, unlike the full-information feedback setting, the Ldim alone cannot provide matching lower and upper bounds on the minimax expected number of mistakes under apple tasting feedback. Indeed,  for the simple class of singletons over the natural numbers, $\mathcal{H}_{\text{sing}} := \{x \mapsto \mathbbm{1}\{x = a\}: a \in \mathbbm{N}\}$ \, we have that $\operatorname{L}(\Hcal_{\text{sing}}) = 1$ while  the minimax expected number of mistakes scales with the time horizon $T$ (see Section \ref{sec:lbrandomreal}). On the other hand, for the ``flip" of the singletons, $\mathcal{H} = \{x \mapsto \mathbbm{1}\{x \neq a\}: a \in \mathbbm{N}\}$, we also have that $\operatorname{L}(\Hcal) = 1$, but $\Hcal$ is trivially learnable in at most $1$ mistake in the realizable setting. Accordingly, new ideas are needed to handle the asymmetric nature of apple tasting feedback.  
 
 As a first step, we go beyond the symmetric nature of complete binary trees and define a new asymmetric binary tree called an \emph{apple tree}. In particular, a binary tree $\mathcal{T}$ of depth $d$ and width $w$ is an apple tree if it admits the following recursive structure. An apple tree of width $w \geq d$ is a complete binary tree with depth $d$. An apple tree with width $w = 1$ and depth $d$ is a degenerate binary tree of depth $d$ where every internal node has only a left child. An apple tree $\mathcal{T}(w, d)$ of depth $d$ and width $w < d$ has a root node $v$  whose left subtree is an apple tree $\mathcal{T}(w, d-1)$  and whose right subtree is an apple tree $\mathcal{T}(w-1, d-1)$. At a high-level, the width of an apple tree $w$ controls the number of ones any path starting from the root can have before the path ends. The depth $d$ of an apple tree controls the maximum number of zeros along any path starting from the root. From this perspective, one can alternatively construct an apple tree of width $w$ and depth $d$ by starting with a complete binary tree of depth $d$ and then trimming each path starting from the root node such that it ends once it contains $w$ ones or until a leaf node has been reached. 
 See Figure \ref{fig:at} for some examples of apple trees.

Similar to Littlestone trees, we can label the internal nodes of an apple tree with instances in $\mathcal{X}$ and the edges with elements of $\{0, 1\}$. By doing so, we get an Apple Littlestone (AL) tree. 

\begin{definition}[Apple Littlestone tree]
\noindent An \emph{Apple Littlestone tree} of width $w$ and depth $d$ is an \emph{apple} tree of width $w$ and depth $d$ where the internal nodes are labeled by instances of $\mathcal{X}$ and the left and right outgoing edges from each internal node are labeled by $0$ and $1$ respectively.
\end{definition}

\begin{figure}[t]
\includegraphics[width=15cm]{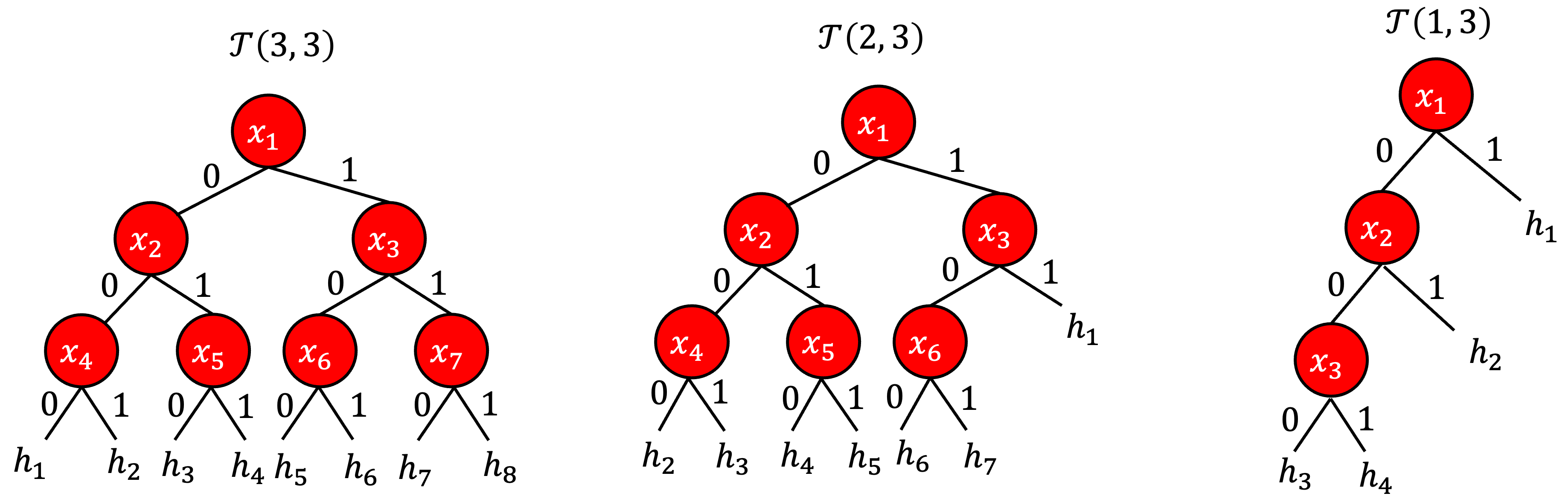}
\centering
\caption{Shattered Apple Littlestone trees of (width, depth):  $(3, 3)$ (left), $(2, 3)$ (middle), and $(1,3)$ (right).}
\label{fig:at}
\end{figure}

The notion of shattering for Littlestone trees extends exactly to AL trees. Formally, an AL tree $\mathcal{T}(w, d)$ of width $w$ and depth $d$ is shattered by $\mathcal{H}$, if for every path $\sigma$ down the tree $\mathcal{T}$, there exists a hypothesis $h_{\sigma} \in \mathcal{H}$ consistent with $\{(x_i, \sigma_i)\}_{i=1}^{|\sigma|}$. Note that, unlike Littlestone trees, AL trees are \emph{imbalanced}.  In fact, for an AL tree $\mathcal{T}$ of width $w$ and depth $d$, there can be at most $w$ ones along any valid path $\sigma$ down the tree before the path ends. Therefore, not all root-to-leaf paths are of the same length. Nevertheless, this notion of shattering is still well defined and naturally leads to a combinatorial dimension analogous to the Littlestone dimension.

\begin{definition}[Apple Littlestone dimension]
    \noindent The \emph{Apple Littlestone dimension} of $\mathcal{H}$ at width $w \in \mathbbm{N}$, denoted $\operatorname{AL}_{w}(\mathcal{H})$, is the largest $d$ such that there exists an apple tree $\mathcal{T}(w, d)$ of width $w$ and depth $d$ shattered by $\mathcal{H}$. If there exists shattered Apple Littlestone trees $\mathcal{T}$ with width $w$  of arbitrarily large depth, then we say that $\operatorname{AL}_{w}(\mathcal{H}) = \infty$. If there are no shattered apple trees $\mathcal{T}$ of width $w$, then we say that $\operatorname{AL}_{w}(\mathcal{H}) = 0$.
\end{definition}


In general, the value of $\operatorname{AL}_{w}(\mathcal{H})$ for $w \leq \operatorname{L}(\mathcal{H})$  can be much larger than $\operatorname{L}(\mathcal{H})$. For example, even for the class of singletons defined over $\mathbb{N}$, we have that $\operatorname{AL}_{1}(\mathcal{H}_{\text{sing}}) = \infty$ while $\operatorname{L}(\mathcal{H}_{\text{sing}}) = 1$. Accordingly, unlike the Ldim, the Apple Littlestone dimension (ALdim), does not provide a qualitative characterization of learnability. Instead, using the ALdim, we define a new combinatorial parameter termed the Effective width. In Section \ref{sec:real}, we show that the Effective width provides a tight \emph{quantitative characterization} of realizable learnability under apple tasting feedback.

\begin{definition}[Effective width]
\noindent The \emph{Effective width} of a hypothesis class $\mathcal{H}$, denoted  $\operatorname{W}(\mathcal{H})$, is the smallest $w \in \mathbbm{N}$ such that $\operatorname{AL}_{w}(\mathcal{H}) < \infty$. If there is no $w \in \mathbbm{N}$ such that $\operatorname{AL}_{w}(\mathcal{H}) < \infty$, then we say that $\operatorname{W}(\mathcal{H}) = \infty$.  
\end{definition}

\noindent The following lemma, whose proof is in Appendix \ref{app:struct}, establishes important properties of  $\operatorname{AL}_w(\Hcal)$ and $\operatorname{W}(\Hcal)$ that we use to  characterize learnability. 


\begin{lemma}[Structural Properties] \label{lem:struct} For every $\Hcal \subseteq \{0, 1\}^{\Xcal}$, the following statements are true.
\begin{enumerate}
\item[(i)] $\operatorname{AL}_{w_1}(\Hcal) \geq \operatorname{AL}_{w_2}(\Hcal)$ for all $w_1 < w_2.$
\item[(ii)] $\operatorname{AL}_{w}(\mathcal{H}) \geq \min\{w, \operatorname{L}(\mathcal{H})\}.$
\item[(iii)] $\operatorname{AL}_{w}(\Hcal) = \operatorname{L}(\mathcal{H})$ for all $w \geq \operatorname{L}(\Hcal) + 1$ when $\operatorname{L}(\Hcal) < \infty.$
\item[(iv)] $\operatorname{W}(\Hcal) \leq \operatorname{L}(\Hcal) + 1$ when $\operatorname{L}(\Hcal) < \infty.$
\item[(v)] $\operatorname{W}(\mathcal{H}) < \infty \Longleftrightarrow \operatorname{L}(\mathcal{H}) < \infty.$
\end{enumerate}
\end{lemma}

Property (iv) can be tight in the sense that for the class of singletons, $\operatorname{W}(\mathcal{H}_{\text{sing}}) = 2$ while $\operatorname{L}(\Hcal_{\text{sing}}) = 1$. Moreover, one cannot hope to  lower bound $\operatorname{W}(\mathcal{H})$ in terms of $\operatorname{L}(\mathcal{H})$. Indeed, for any \emph{finite} hypothesis class $\mathcal{H}$, we have that $\operatorname{W}(\mathcal{H}) = 1$ while $\operatorname{L}(\mathcal{H})$ can be made arbitrarily large. Finally, as an example, we also compute the Effective width for the $k$-wise generalization of $\Hcal_{\text{sing}}$ in Appendix \ref{app:examples}.

\section{Realizable Learnability} \label{sec:real}

In this section, we revisit the learnability of apple tasting in the realizable setting, first studied by \cite{helmbold2000apple}. Our main result is Theorem \ref{thm:real}, which lower- and upper bounds the minimax expected number of mistakes in terms of the Littlestone dimension and the Effective width.  

\begin{theorem}[Realizable Learnability] 
\label{thm:real}
\noindent For any hypothesis class $\mathcal{H} \subseteq \{0, 1\}^{\mathcal{X}}$, 
$$ \frac{1}{8} \min\Bigl\{\max\Bigl\{\sqrt{(\operatorname{W}(\mathcal{H})-1) \, T}, \operatorname{L}(\mathcal{H})\Bigl\}, T\Bigl\} \leq \inf_{\mathcal{A}}\operatorname{M}_{\mathcal{A}}(T, \mathcal{H})\leq \operatorname{AL}_{\operatorname{W}(\mathcal{H})}(\mathcal{H})  + 2\sqrt{(\operatorname{W}(\mathcal{H})-1)T}.$$
\end{theorem}

The lower and upper bounds of Theorem \ref{thm:real} can be tight up to constant factors. There are two cases to consider.  When $\operatorname{W}(\mathcal{H}) = 1$, the lower and upper bounds in Theorem \ref{thm:real} reduce to  $\frac{\operatorname{L}(\mathcal{H}) }{8} \leq \inf_{\mathcal{A}}\operatorname{M}_{\mathcal{A}}(T, \mathcal{H})\leq \operatorname{AL}_{1}(\mathcal{H})$ for $T \geq \operatorname{L}(\mathcal{H})$. Taking $|\mathcal{X}| = d < \infty$ and $\mathcal{H} = \{0, 1\}^{\mathcal{X}}$ gives that $\operatorname{L}(\mathcal{H}) = \operatorname{AL}_{1}(\mathcal{H}) = d$, ultimately implying that the lower- and upper bounds can be off by only a constant factor of $\frac{1}{8}$. Secondly, consider the case where $\operatorname{W}(\mathcal{H}) \geq 2$. Then, if $T \geq \max \{\operatorname{W}(\mathcal{H}) - 1, \operatorname{AL}_{\operatorname{W}(\mathcal{H})}^2(\mathcal{H})\}$, Theorem \ref{thm:real} implies that $ \frac{1}{8} \sqrt{(\operatorname{W}(\mathcal{H})-1) \, T} \leq \inf_{\mathcal{A}}\operatorname{M}_{\mathcal{A}}(T, \mathcal{H})\leq 3\sqrt{(\operatorname{W}(\mathcal{H})-1)T}$, showing that the upper- and lower bounds are off only by a constant factor. 

 Theorem \ref{thm:real} implies that when $\operatorname{W}(\mathcal{H}) = 1$, a constant upper bound on the expected regret is possible. In fact, when $\operatorname{AL}_1(\mathcal{H}) < \infty$, there exists a \emph{deterministic} online learner which makes at most $\operatorname{AL}_1(\mathcal{H}) $ mistakes in the realizable setting under apple tasting feedback (see Appendix \ref{sec:updetreal}). On the other hand, Theorem \ref{thm:real} also shows that, in full generality, it is not possible to achieve a constant \emph{expected} mistake bound under apple tasting feedback in the \emph{realizable} setting. Indeed, if $\operatorname{W}(\mathcal{H}) > 1$, then the worst-case expected mistakes of any  randomized learner, under apple tasting feedback, is at least $\Omega(\sqrt{T})$. This is in contrast to the full-information setting, where the minimax expected number of mistakes in the realizable setting is constant, and that too achieved by a \emph{deterministic} learner (i.e SOA). Accordingly, Theorem \ref{thm:real} gives a trichotomy in the minimax expected number of mistakes for the realizable setting. 

\begin{corollary} [Trichotomy in minimax expected number of mistakes] \label{cor:trichotomy}
\noindent For any hypothesis class $\mathcal{H} \subseteq \{0, 1\}^{\Xcal}$, 
$$
\inf_{\mathcal{A}} \operatorname{M}_{\mathcal{A}}(T, \mathcal{H})=\begin{cases}
			\Theta(1), & \text{if $\operatorname{W}(\mathcal{H}) = 1.$}\\
            \Theta(\sqrt{T}) & \text{if $2 \leq \operatorname{W}(\mathcal{H}) < \infty.$}\\
            \Theta(T), & \operatorname{W}(\mathcal{H}) =  \infty.
		 \end{cases}
$$
\end{corollary}

In Section \ref{sec:agn}, we will show that $\inf_{\mathcal{A}}\operatorname{R}_{\mathcal{A}}(T, \mathcal{H}) = \tilde{\Theta}(\sqrt{T})$, where $\tilde{\Theta}$ hides poly-logarithmic factors in $T$. With this in mind, Corollary \ref{cor:trichotomy} shows that when $\operatorname{W}(\mathcal{H}) \geq 2$, realizable learnability under apple tasting feedback can be as hard as agnostic learnability under apple tasting feedback. Unfortunately, for many simple classes, like the singletons over $\mathbbm{N}$, we have $\operatorname{W}(\mathcal{H}) \geq 2$.  On the other hand, for classes containing hypothesis that rarely output $0$, like the ``flip'' of the class of singletons, realizable learnability under apple tasting feedback can be as easy as realizable learnability under full-information feedback.




\subsection{Upper Bounds for Randomized Learners in the Realizable Setting} \label{sec:uprandomreal}

We prove a slightly stronger upper bound than the one stated in Theorem \ref{thm:real}.

\begin{lemma}[Randomized Realizable Upper Bound] \label{lem:up}
\noindent For any hypothesis class $\mathcal{H} \subseteq \{0, 1\}^{\mathcal{X}}$,
$$\inf_{\mathcal{A}}\operatorname{M}_{\mathcal{A}}(T, \mathcal{H}) \leq \inf_{w \in \mathbb{N}} \left\{\operatorname{AL}_w(\mathcal{H}) + 2\sqrt{(w-1)T} \right\}.
%
$$
\end{lemma}

The upper bound in Theorem \ref{thm:real} follows by picking $w = \operatorname{W}(\mathcal{H})$.  If one picks $w = \operatorname{L}(\mathcal{H})+1$, then $\operatorname{AL}_w(\mathcal{H}) = \operatorname{L}(\mathcal{H})$ and we get an upper bound of $3\sqrt{\operatorname{L}(\mathcal{H})T}$ on the expected mistakes.  

Lemma \ref{lem:up} follows from composing the next two lemmas. Lemma \ref{lem:soaconstrained} shows that if $\operatorname{AL}_w(\mathcal{H}) < \infty$, then there exists a deterministic online learner, under \emph{full-information} feedback, that makes at most $w-1$ false negative mistakes and at most $\operatorname{AL}_w(\mathcal{H})$ false positive mistakes. Lemma \ref{lem:conver} is from \cite{helmbold2000apple} and shows how to convert any online learner under full-information feedback into an online learner under apple tasting feedback. 

\begin{lemma} \label{lem:soaconstrained}
\noindent For any $\mathcal{H} \subset \{0, 1\}^{\Xcal}$ and $w \in \mathbbm{N}$ such that $\operatorname{AL}_w(\mathcal{H}) < \infty$, there exists a deterministic online learner which, under \emph{full-information} feedback, makes at most $w-1$ \emph{false negative} mistakes and at most $\operatorname{AL}_w(\mathcal{H})$ \emph{false positive} mistakes in the realizable setting. 
\end{lemma}

\begin{proof}
Suppose $w \in \mathbb{N}$ such that $ \operatorname{AL}_w(\mathcal{H}) < \infty$ and denote $\mathcal{A}$ to be Algorithm \ref{alg:real_full}.

\begin{algorithm}
\setcounter{AlgoLine}{0}
\caption{Realizable Algorithm Under Full-Information Feedback}\label{alg:real_full}
\KwIn{$V_1 = \mathcal{H}$ , pick $w_1 = w$ such that $\operatorname{AL}_{w}(\mathcal{H}) < \infty$}

\For{$t = 1,...,T$} {
    Receive $x_t$.

     For  each $y \in \{0,1\} $, define $V_{t}^{y} = \{h \in V_{t} \, \mid h(x_t) = y\}$. 

     \uIf{$V_t(x_t) = \{y\}$}{

        Predict $\hat{y}_t = y$.
     
     }

     \uElse{
         If $|V_{t}^1| \geq 1$, and  $\operatorname{AL}_{w_{t}}(V_{t}^0) < \operatorname{AL}_{w_{t}}(V_{t})$, predict $\hat{y}_t = 1$. Otherwise, predict $\hat{y}_t = 0$.
     }

    Receive $y_t$ and update $V_t \leftarrow V_t^{y_t}$. 
    
    If $\hat{y}_t =0$ and $y_t =1$, then update $w_{t+1} \leftarrow w_{t}-1$. Else, set $w_{t+1} \leftarrow w_{t}$. 
     
}
\end{algorithm}

Let $(x_1, y_1), ..., (x_T, y_T)$ be the stream to be observed by $\mathcal{A}$. We show that $\mathcal{A}$, initialized at $w_1 =w$, makes at most $\operatorname{AL}_w(\mathcal{H})$ false positive mistakes and at most $w-1$ false negative mistakes.  Let $S_{+} = \{ t \in [T] \mid \hat{y}_t =1 \, \text{ and }\,  y_t =0 \}$ be the set of time points where $\mathcal{A}$ makes false positive mistakes, and $S_{-} = \{ t \in [T] \mid \hat{y}_t =0 \, \text{ and }\,  y_t =1 \}$ be the set of time points where $\mathcal{A}$  makes false negative mistakes. We show $ |S_{+}| \leq \operatorname{AL}_w(\mathcal{H})$ by first establishing
\begin{equation}\label{eq:falsepos}
    \operatorname{AL}_{w_{t+1}}(V_{t+1}) \leq  \operatorname{AL}_{w_{t}}(V_{t}) - \indicator\{t \in S_{+}\}, \quad \quad \forall t \in [T].
\end{equation}
This inequality then implies that the number of false positive mistakes of $\mathcal{A}$ is
\begin{align*}
\sum_{t=1}^T \indicator\{t \in S_{+}\} &\leq \sum_{t=1}^{T} \left(\operatorname{AL}_{w_{t}}(V_{t}) - \operatorname{AL}_{w_{t+1}}(V_{t+1}) \right)  \\
&= \operatorname{AL}_{w_{1}}(V_{1}) - \operatorname{AL}_{w_{T+1}}(V_{T+1})\\
&\leq \operatorname{AL}_{w_{1}}(V_{1}) = \operatorname{AL}_{w}(\mathcal{H}).
\end{align*}
To prove inequality \eqref{eq:falsepos}, we consider the two cases:  $t \in S_{+}$ and  $t \notin S_{+}$. Suppose $t \in S_{+}$. Then, we know that $\hat{y}_t=1$ and by the prediction rule of $\mathcal{A}$, we must have $\operatorname{AL}_{w_{t}}(V_{t}^0) < \operatorname{AL}_{w_{t}}(V_{t})$. Since $y_t =0$, we further obtain that $V_{t+1} = V_t^0$ and $w_{t+1} = w_{t}$ in this case.  This yields $\operatorname{AL}_{w_{t+1}}(V_{t+1}) < \operatorname{AL}_{w_{t}}(V_{t}) $, which subsequently implies $\operatorname{AL}_{w_{t+1}}(V_{t+1}) \leq  \operatorname{AL}_{w_{t}}(V_{t}) - \indicator\{t \in S_{+}\} $.

Now, let us consider the case when $t \notin S_{+}$. In the case when $t \notin S_{+} \cup S_{-}$,  we have $w_{t+1} = w_t$ and $\indicator\{t \in S_{+}\}=0$. Thus, we trivially obtain
$ \operatorname{AL}_{w_{t+1}}(V_{t+1}) \leq  \operatorname{AL}_{w_{t}}(V_{t})- \indicator\{t \in S_{+}\}$ since $V_{t+1} \subseteq V_t$. Next, let us consider the case when $t \in S_{-}$. In this case, we have $w_{t+1} = w_{t}-1$,  $V_t = V_t^{1}$, and $\indicator\{t \in S_{+}\} =0$. Thus, to establish inequality \eqref{eq:falsepos}, it suffices to show that $\operatorname{AL}_{w_{t}-1}(V_{t}^1) \leq  \operatorname{AL}_{w_{t}}(V_{t})$. Suppose, for the sake of contradiction, this is not true and we instead have $\operatorname{AL}_{w_{t}-1}(V_{t}^1) >  \operatorname{AL}_{w_{t}}(V_{t})$. Let $ d:= \operatorname{AL}_{w_{t}}(V_{t})$. Note that $d > 0$ because there must exist $h_1, h_2 \in V_t$ such that $h_1(x_t) \neq h_2(x_t)$ or otherwise $\mathcal{A}$ would not have made a false negative mistake. Since $\operatorname{AL}_{w_{t}-1}(V_{t}^1) > d$, we are guaranteed the existence of an AL tree $\mathcal{T}_{1}(w_{t}-1, d)$ shattered by $V_t^1$. Furthermore, as $\hat{y}_t =0$ and $|V_{t}^1|\geq 1$, the prediction rule implies that $ \operatorname{AL}_{w_{t}}(V_{t}^0) \geq   \operatorname{AL}_{w_{t}}(V_{t}) = d$. Accordingly, we are also guaranteed the existence of an AL tree $\mathcal{T}_{0}(w_{t}, d)$ shattered by $V_t^0$. Now consider an AL tree $\mathcal{T}$ that has $x_t$ in its root-node, has a subtree $\mathcal{T}_{0}(w_{t}, d)$ attached to left-outgoing edge from the root-node and has a subtree $\mathcal{T}_{1}(w_{t}-1, d)$ attached to right-outgoing edge from the root-node. Since all hypotheses in $V_t^0$ output $0$ on $x_t$ and all hypotheses in $V_t^1$ output $1$ on $x_t$, the tree $\mathcal{T}$ shattered by $V_t$. Since $\mathcal{T}$ is a valid AL tree of width $w_t$ and depth $d+1$, we have that $  \operatorname{AL}_{w_{t}}(V_t) \geq d+1$, a contradiction to our assumption that $\operatorname{AL}_{w_{t}}(V_t) = d$. Therefore, we must have  $\operatorname{AL}_{w_{t}-1}(V_{t}^1) \leq  \operatorname{AL}_{w_{t}}(V_{t})$ when $t \in S_{-}.$

Next, we show that $\mathcal{A}$ makes at most $w-1$ false negative mistakes. Let $t^{\star} \in [T]$ be the time point where the algorithm makes its $(w-1)$-th false negative mistake. If such time point $t^{\star}$ does not exist, then we trivially have $|S_{-}| \leq w-2 < w-1$. We now consider the case when $t^{\star} \in [T]$ exists. It suffices to show that, $\forall t > t^{\star}$, we have $t \notin S_{-}$. Suppose, for the sake of contradiction, $\exists \, t > t^{\star}$ such that $t \in S_{-}$. Since $\hat{y}_t = 0$ and $y_t =1$, we must have $|V_t^1| \geq 1$. Thus, the prediction strategy implies that $\operatorname{AL}_{w_{t}}(V_{t}^0) \geq  \operatorname{AL}_{w_{t}}(V_{t}) $. Given that $t > t^{\star}$ and $\mathcal{A}$ has already made $w-1$ false negative mistakes, we must have $w_t = 1$. Thus, we have $\operatorname{AL}_{1}(V_{t}^0) \geq  \operatorname{AL}_{1}(V_{t}) =: d$. Note that $d \geq 1$ because there must exist $h_1, h_2 \in V_t$ such that $h_1(x_t) \neq h_2(x_t)$. Since $\operatorname{AL}_{1}(V_{t}^0)  \geq d$, we are guaranteed the existence of an AL tree $\mathcal{T}_0(1, d)$ of width $1$ and depth $d$ shattered by $V_t^0$. Next, consider a tree $\mathcal{T}$ with $x_t$ on the root node and has a subtree $\mathcal{T}_0(1, d)$ attached to the left-outgoing edge from the root node. Let $h \in V_t$ any hypothesis such that $h(x_t) =1 $. The hypothesis $h$ must exist because $|V_t^1| \geq 1$. By putting $h $ in the leaf node following the right-outgoing edge from the root node in $\mathcal{T}$, it is clear that $\mathcal{T}$ is a valid AL tree of width $1$ and depth $d+1$ shattered by $V_t$. The existence of $\mathcal{T}$ implies that  $ \operatorname{AL}_{1}(V_{t}) \geq d+1$, a contradiction to our assumption $\operatorname{AL}_{1}(V_{t}) =d $. Thus,  $\forall t> t^{\star}$, we have $t \notin S_{-}$. Therefore, $\mathcal{A}$ makes no more than $w-1$ false negative mistakes. \end{proof}

We remark that \cite{helmbold2000line} also give a deterministic online learner in the full-information setting under constraints on the number of false positive and false negative mistakes (see Algorithm SCS in \cite[Section 2]{helmbold2000line}). However, similar to \cite{helmbold2000apple}, their algorithm checks the existence of an online learning algorithm satisfying certain properties. We extend on this result by giving an SOA-type algorithm that only requires computing combinatorial dimensions. 

Lemma \ref{lem:conver} is the restatement of Corollary 2 in \cite{helmbold2000apple}. For completeness sake, we provide a proof in  Appendix \ref{app:conv}.  Lemma \ref{lem:up} follows by composing Lemma \ref{lem:soaconstrained} and Lemma \ref{lem:conver}.

\begin{lemma} [\cite{helmbold2000apple}]\label{lem:conver}
\noindent For any $\mathcal{H} \subseteq \{0, 1\}^{\mathcal{X}}$, if there exists a deterministic learner which, under full-information feedback, makes at most $M_{-}$ \emph{false negative} mistakes and at most $M_{+}$ \emph{false positive} mistakes, then there exists a randomized learner, whose expected number of mistakes, under apple tasting feedback, is at most $M_{+} + 2\sqrt{T M_{-}}$ in the realizable setting. 
\end{lemma}

\subsection{Lower Bounds for Randomized Learners in the Realizable Setting} \label{sec:lbrandomreal}

As in the upper bound, we prove a slightly stronger lower bound than the one stated in Theorem \ref{thm:real}.

\begin{lemma}[Realizable Lower Bound] \label{lem:reallb}
\noindent For any hypothesis class $\mathcal{H} \subseteq \{0, 1\}^{\Xcal}$, 
$$\inf_{\mathcal{A}}\operatorname{M}_{\mathcal{A}}(T, \mathcal{H}) \geq \frac{1}{8}\sup_{w \in \mathbbm{N}} \sqrt{\min\{w, \operatorname{L}(\mathcal{H}), T\} \,\min\{\operatorname{AL}_w(\mathcal{H}), T\}}.$$
\end{lemma}

The lower bounds in Theorem \ref{thm:real} follows by picking $w = \operatorname{W}(\mathcal{H}) - 1$ and $w = \operatorname{L}(\mathcal{H}) + 1$ respectively. When $w = \operatorname{W}(\mathcal{H}) - 1$, we have that 
$\min\{w, \operatorname{L}(\mathcal{H}), T\} = \min\{\operatorname{W}(\mathcal{H}) - 1, T\}$ and  $\min\{\operatorname{AL}_w(\mathcal{H}), T\} \geq \min\{\operatorname{W}(\mathcal{H}) - 1, T\}$ using Lemma \ref{lem:struct} (ii) and (iv).
On the other hand, when $w = \operatorname{L}(\mathcal{H}) + 1$, we have that $\min\{\operatorname{AL}_w(\mathcal{H}), T\} = \min\{\operatorname{L}(\mathcal{H}), T\}$ using Lemma \ref{lem:struct} (iii).


\begin{proof} Let $\mathcal{H} \subseteq \{0, 1\}^{\mathcal{X}}$, $w \in \mathbbm{N}$, and $T \in \mathbb{N}$ be the time horizon. Since learning under apple tasting feedback implies learning under full-information feedback, a lower bound of $\frac{\min\{T, \operatorname{L}(\mathcal{H})\}}{2}$ on the minimax expected number of mistakes follows trivially from the full-information feedback lower bound. Accordingly, for the remainder of the proof we suppose $w \leq \min\{\operatorname{L}(\mathcal{H}), T\}$, since if this condition is not met, the claimed lower bound is at most $\frac{\min\{T, \operatorname{L}(\mathcal{H})\}}{2}$.   
Let $\mathcal{T}$ be any AL tree of width $w$ of depth $d = \floor*{\sqrt{w\min\{T, \operatorname{AL}_{w}(\mathcal{H})\}}}$ shattered by $\mathcal{H}$. Such a tree must exist because $d \leq \operatorname{AL}_{w}(\mathcal{H})$. Let $\mathcal{A}$ be any randomized apple tasting online learner. Our goal will be to construct a hard, determinsitic,  realizable stream of instances $(x_1, y_1), ..., (x_T, y_T)$ such that $\mathcal{A}$'s expected regret is at least $\frac{d}{4}$. 




We first construct a path $\sigma^{\star}$ down $\mathcal{T}$ recursively using $\mathcal{A}$. Starting with $\sigma^{\star}_1$, let $A_1$ be the event that $\mathcal{A}$, if presented with $\floor*{\frac{d}{w}}$ copies of the root node $x^{\star}_1$, predicts $1$ on at least one of the copies. Then, set $\sigma^{\star}_1 = 0$ if $\mathbbm{P}(A_1) \geq \frac{1}{2}$ and set  $\sigma^{\star}_1 = 1$ otherwise. For $j \geq 2$, let $x^{\star}_1, ..., x^{\star}_j$ be the sequence of instances labeling the internal nodes along the prefix $(\sigma^{\star}_1, ..., \sigma^{\star}_{j-1})$ down $\mathcal{T}$. Let $A_j$ be the event that $\mathcal{A}$, if simulated with the sequence of $(j-1)\floor*{\frac{d}{w}}$ labeled instances consisting of $\floor*{\frac{d}{w}}$ copies of the labeled instance $(x^{\star}_1, \sigma^{\star}_{1})$, followed by  $\floor*{\frac{d}{w}}$ copies of the labeled instance $(x^{\star}_2, \sigma^{\star}_{2})$,..., followed by $\floor*{\frac{d}{w}}$ copies of the labeled instance $(x^{\star}_{j-1}, \sigma^{\star}_{j-1})$, predicts the label $1$ at least once when presented with $\floor*{\frac{d}{w}}$ copies of the instance $x^{\star}_j$. Set $\sigma^{\star}_j = 0$ if $\mathbbm{P}(A_j) \geq \frac{1}{2}$ and set  $\sigma^{\star}_j = 1$ otherwise. Continue this process until $\sigma^{\star}$ is a valid path that reaches the end of tree $\mathcal{T}$.

We now construct our hard, labeled stream in blocks of size $\floor*{\frac{d}{w}}$. Each block only contains a single labeled instance, repeated $\floor*{\frac{d}{w}}$ times. For the first block $B_1$, repeat the labeled instance $(x_1^{\star}, \sigma^{\star}_1)$. Likewise, for block $B_j$ for $2 \leq j \leq |\sigma^{\star}|$, repeat for $\floor*{\frac{d}{w}}$ times the labeled instance  $(x_j^{\star}, \sigma^{\star}_j)$. Now, consider the stream $S = (B_1, ..., B_{|\sigma^{\star}|})$ obtained by concatenating the blocks $B_1, ..., B_{|\sigma^{\star}|}$ in that order. If $|\sigma^{\star}|\floor*{\frac{d}{w}} < T$, populate the rest of the stream $S$ with the labeled instance $(x^{\star}_{|\sigma^{\star}|}, \sigma^{\star}_{|\sigma^{\star}|}).$

We first claim that such a stream is realizable by $\mathcal{H}$. This follows trivially from the fact that (1) $\sigma^{\star}$ is a valid path down $\mathcal{T}$, (2) by the definition of shattering, there exists a hypothesis $h \in \mathcal{H}$ such that for all $j \in [|\sigma^{\star}|]$, we have $h(x_j^{\star}) = \sigma_{j}^{\star}$ and (3) our stream $S$ only contains labeled instances from the set $\{(x_j^{\star}, \sigma_j^{\star})\}_j$. We now claim that $\mathcal{A}$'s expected regret on the stream $S$ is at least $\frac{d}{4}$. To see this, observe that whenever $\sigma_j^{\star} = 1$, $\mathcal{A}$'s expected mistakes on the block $B_j$ is at least $\frac{1}{2}\floor*{\frac{d}{w}}$ since $\mathcal{A}$ gets passed the labeled instance $(x_j^{\star}, 1)$ for $\floor*{\frac{d}{w}}$ iterations, but the probability that it never predicts $1$ on this batch after seeing $B_1, ..., B_{j-1}$ is $\mathbb{P}(A^c_j) \geq \frac{1}{2}$.  Likewise, whenever $\sigma_j^{\star} = 0$, $\mathcal{A}$'s expected mistakes on the block $B_j$ is at least $\frac{1}{2}$ since it gets passed the labeled instance $(x_j^{\star}, 0)$ for $\floor*{\frac{d}{w}}$ time points but predicts $1$ on at least one of them with probability $\mathbb{P}(A_j) \geq \frac{1}{2}$.

We now lower bound the expected mistakes of $\mathcal{A}$ on the entire stream $S$ by considering the number of ones in $\sigma^{\star}$ on a case by case basis. Note that since $\sigma^{\star}$ is a valid path down $\mathcal{T}$, we have $w \leq |\sigma^{\star}| \leq d$. Consider the case where $\sigma^{\star}$ has $w$ ones. Then, $\mathcal{A}$'s expected regret is at least its expected regret on those batches $B_j$ where $\sigma^{\star}_j = 1$. Thus, its expected regret is at least $\frac{w}{2}\floor*{\frac{d}{w}} \geq \frac{w}{2}\frac{d}{2w} \geq \frac{d}{4}$. Consider the case where $\sigma^{\star}$ has $w-j$ ones for $w \geq j \geq 1$. Then, since $\sigma^{\star}$ is a valid path, it must be the case that there are $d - (w - j)$ zero's in $\sigma^{\star}$. Therefore, $\mathcal{A}$'s expected regret is at least
$$\frac{(w - j)}{2}\floor*{\frac{d}{w}} + \frac{d - w + j}{2} \geq \frac{d}{2} - \frac{w-j}{2} + \frac{w-j}{2}\floor*{\frac{d}{w}}\geq  \frac{d}{2}.$$
\noindent where the last inequality follows from the fact that $d \geq w$. Thus, in all cases,  $\mathcal{A}$'s expected regret is at least $\frac{d}{4}.$ The claimed lower bound follows by using the fact that $d\geq \sqrt{w\min\{T, \operatorname{AL}_{w}(\mathcal{H})\}}/2$. 
\end{proof}

\section{Agnostic Learnability} \label{sec:agn}

 We show that the Ldim quantifies the minimax expected regret in the agnostic setting under apple tasting feedback, closing the open problem posed by \cite[Page 138]{helmbold2000apple}. 

\begin{theorem}[Agnostic Learnability] \label{thm:agn}
\noindent 
    For any hypothesis class $\mathcal{H} \subseteq \{0, 1\}^{\mathcal{X}}$, 
    $$\sqrt{\frac{\operatorname{L}(\mathcal{H}) T}{8}} \leq \inf_{\mathcal{A}}\operatorname{R}_{\mathcal{A}}(T, \mathcal{H}) \leq 3\sqrt{\operatorname{L}(\mathcal{H})T\ln T}.$$
\end{theorem}

The lower bound in Theorem \ref{thm:agn} follows directly from the full-information lower bound in the agnostic setting \citep{ben2009agnostic}. Therefore, in this section, we only focus on proving the upper bound. Our strategy will be in two steps. First, we modify the celebrated Randomized Exponential Weights Algorithm (REWA) \citep{cesa2006prediction} to handle apple tasting feedback by using the ideas from \cite{alon2015online}. In particular, our algorithm EXP4.AT is an adaptation of EXP3.G from \cite{alon2015online} to binary prediction with expert advice under apple tasting feedback. Second, we give an agnostic online learner which uses the SOA to construct a finite set of experts that exactly covers  $\mathcal{H}$ and then runs EXP4.AT using these experts. The upper bound in Theorem \ref{thm:agn} follows immediately from the composition of these two results.

\subsection{The EXP4.AT Algorithm}

In this subsection, we present EXP4.AT, an adaptation of REWA to handle apple tasting feedback.

\begin{algorithm}
\caption{EXP4.AT: online learning with apple tasting feedback}\label{alg:cap}
\KwIn{Learning rate $\eta \in (0, \frac{1}{2})$}

Let $q_1$ be the uniform distribution over $[N]$

\For{$t = 1,...,T$} {
    Get advice $\mathcal{E}^1_t,...,\mathcal{E}^N_t \in \{0, 1\}^N$
    
    Compute $p_t^1 = (1-\eta)\sum_{i=1}^N q_t^i \mathcal{E}^i_t + \eta$

    Predict $\hat{y}_t = 1$ with probability $p_t^1$ and  $\hat{y}_t = 0$ with probability $p_t^0 = 1 - p_t^1$

    Observe true label $y_t$ if $\hat{y}_t = 1$ and let $\hat{\ell}_t(y) = \frac{\mathbbm{1} \{y \neq y_t\}\mathbbm{1} \{\hat{y}_t = 1\}}{p_t^1}$
    
    For $i = 1,...,N$ update
    $q_{t+1}^i = \frac{q_t^i\text{exp}(-\eta \hat{\ell}_t(\mathcal{E}_t^i))}{\sum_{j=1}^N q_t^j\text{exp}(-\eta \hat{\ell}_t(\mathcal{E}_t^j))}$
    
     
}
    
\end{algorithm}

\begin{theorem}[EXP4.AT Regret Bound] \label{thm:exp4rb}
\noindent If $\eta = \sqrt{\frac{\ln{N}}{2T}}$, then for any sequence of true labels $y_1, ..., y_T$, the predictions $\hat{y}_1, ..., \hat{y}_T$, output by \emph{EXP4.AT} satisfy: 

$$\mathbbm{E}\left[\sum_{t=1}^T \mathbbm{1}\{\hat{y}_t \neq y_t\} \right] \leq \inf_{j \in [N]} \sum_{t=1}^T \mathbbm{1}\{\mathcal{E}^j_t \neq y_t\} + 3\sqrt{T\ln{N}}.$$

\end{theorem}

In order to prove Theorem \ref{thm:exp4rb}, we need the following lemma which gives a second-order regret bound for the EXP4.AT algorithm. The proof of Lemma \ref{lem:exp4rb} follows a similar potential-function strategy as in the proof of Lemma 4 in \cite{alon2015online} and can be found in Appendix \ref{app:experts}.

\begin{lemma}[EXP4.AT Second-order Regret Bound] \label{lem:exp4rb}
\noindent
For any $\eta \in (0, \frac{1}{2})$ and any sequence of true labels $y_1, ..., y_T$, the probabilities $p_1, ..., p_T$ output by \emph{EXP4.AT} satisfy
$$\sum_{t=1}^T \sum_{y \in \{0, 1\}} p_t^y \hat{\ell}_t(y) - \inf_{j \in [N]}\sum_{t=1}^T \hat{\ell}_t(\mathcal{E}_t^{j} ) \leq \frac{\ln N}{\eta} + \eta \sum_{t=1}^T \hat{\ell}_t(1) + \eta\sum_{t=1}^T p_t^1(1 - p_t^1) \hat{\ell}_t(0)^2 + \eta\sum_{t=1}^T p_t^1 \hat{\ell}_t(1)^2.$$
\end{lemma}

Theorem \ref{thm:exp4rb} follows by taking expectations of both sides of the inequality in Lemma \ref{lem:exp4rb}. The full proof can be found in Appendix \ref{app:thmexp}.

\subsection{Proof Sketch of Theorem \ref{thm:agn}}
Given any hypothesis class $\mathcal{H}$, we construct an agnostic online learner under apple tasting feedback with the claimed upper bound on expected regret. Similar to the generic agnostic online learner in the full-information setting \citep{ben2009agnostic}, the high-level strategy is to use the SOA to construct a small set of experts $E$ such that $|E| \leq T^{\operatorname{L}(\mathcal{H})}$ and for every $h \in \mathcal{H}$, there exists an expert $\mathcal{E}_h \in E$ such that $\mathcal{E}_h(x_t) = h(x_t)$ for all $t \in [T]$. Then, our agnostic online learner will run EXP4.AT using this set of experts $E$. The upper bound in Theorem \ref{thm:agn} immediately follows from the guarantee of EXP4.AT in Theorem \ref{thm:exp4rb} and the fact that we have constructed an exact cover of $\mathcal{H}$. The full proof of Theorem \ref{thm:agn} can be found in Appendix \ref{app:agn}.

\section{Discussion and Open Questions} \label{sec:disc}
In this work, we revisited the classical setting of apple tasting and studied learnability from a combinatorial perspective. Our work makes an important step towards developing learning theory for online classification under partial feedback. An important future direction is to extend this work to multiclass classification under various partial feedback models,  such as those captured by feedback graphs \citep{alon2015online}. 

With respect to apple tasting, there are still interesting open questions. For example, our focus in the realizable setting was on \emph{randomized} learnability. Remarkably, under full-information feedback,  randomness is not needed to design online learners with optimal mistake bounds (up to constant factors). It is therefore natural to ask whether randomness is actually needed in the realizable setting under apple tasting feedback. 

\vspace{2pt}
\noindent \textbf{Question 1.}
\noindent For any $\mathcal{H} \subseteq \{0, 1\}^{\mathcal{X}}$ with $\operatorname{W}(\mathcal{H}) < \infty$, is  $\inf_{\text{Deterministic }\mathcal{A}}\, \, \operatorname{M}_{\mathcal{A}}(T, \mathcal{H}) = o(T)?$


In Appendix \ref{sec:updetreal}, we provide some partial answers. We show that if $\operatorname{W}(\mathcal{H}) = 1$ or $\operatorname{L}(\mathcal{H}) = 1$, then such generic deterministic learners do exist with mistake bounds that are constant factors away from the lower bound in Theorem  \ref{lem:reallb}. We conjecture that the statement in the open question is true.

Our lower and upper bounds in the agnostic setting are matching up to a factor logarithmic in $T$. Recently, \cite{alon2021adversarial} showed that in the full-information setting, this $\log(T)$ factor can be removed from the upper bound, meaning that the optimal expected regret in the agnostic setting under full-information feedback is $\Theta(\sqrt{\operatorname{L}(\mathcal{H}) T})$. As an open question, we ask whether it is possible to also remove the factor of $\log(T)$ from our upper bound in Theorem \ref{thm:agn}. 

\vspace{2pt}
\noindent \textbf{Question 2.}
\noindent For any $\mathcal{H} \subseteq \{0, 1\}^{\mathcal{X}}$, is it true that $\inf_{\mathcal{A}}\operatorname{R}_{\mathcal{A}}(T, \mathcal{H}) = \Theta(\sqrt{\operatorname{L}(\mathcal{H})T})?$

\section*{Acknowledgments} AT acknowledges the support of NSF via grant IIS-2007055. VR acknowledges the support of the NSF Graduate Research Fellowship. US acknowledges the support of the Rackham International Student Fellowship.

\bibliographystyle{plainnat}
\bibliography{arxiv-v3/references}

\appendix

\section{Upper bounds for Deterministic Learners in the Realizable Setting} \label{sec:updetreal}

In this section, we provide deterministic apple tasting learners for some special classes. Our first contribution shows that when $\operatorname{W}(\mathcal{H}) = 1$, there exists deterministic online learner which makes at most $\operatorname{AL}_1(\mathcal{H})$ mistakes under apple tasting feedback. 

\begin{theorem}[Deterministic Realizable upper bound when $\operatorname{W}(\mathcal{H}) = 1$] \label{thm:detrealupAT1}
\noindent For any $\mathcal{H} \subseteq \{0, 1\}^{\mathcal{X}}$, there exists a deterministic online learner which, under apple tasting feedback, makes at most $\operatorname{AL}_1(\mathcal{H})$ mistakes in the realizable setting.
\end{theorem}

\begin{proof}
    We will show that Algorithm \ref{alg:detAT1learner} makes at most  $\operatorname{AL}_1(\mathcal{H})$ mistakes in the realizable setting. 

    \begin{algorithm}\label{alg:detAT1learner}
    \setcounter{AlgoLine}{0}
    \caption{Deterministic Realizable Algorithm For Apple Tasting}\label{alg:detAT1learner}
    \KwIn{$V_1 = \mathcal{H}$} 
    
    \For{$t = 1,...,T$} {
        Receive $x_t$.
        
        If there exists $h \in V_{t}$ such that $h(x_t) = 1$, predict $\hat{y}_t = 1$. Else, predict $\hat{y}_t = 0$.     
    
        If $\hat{y}_t = 1$, receive $y_t$ and update $V_{t+1} \leftarrow \{h \in V_{t}: h(x_t) = y_t\}$ 
    }
        
    \end{algorithm}

    Let $t \in [T]$ be any round such that $\hat{y}_t \neq y_t$. We will show $\operatorname{AL}_{1}(V_{t+1}) \leq \operatorname{AL}_{1}(V_{t}) - 1$. By the prediction strategy and the fact that we are in the realizable setting, if  $\hat{y}_t \neq y_t$ then it must be the case that $\hat{y}_t = 1$ but $y_t = 0$. For the sake of contradiction, suppose that $\operatorname{AL}_{1}(V_{t+1}) = \operatorname{AL}_{1}(V_{t}) = d$. Then, there exists an AL tree $\mathcal{T}$ of width $1$ and depth $d$ shattered by $V_{t+1}$. Consider a new AL tree $\mathcal{T}^{\prime}$ of width $1$ where the root node labeled is $x_t$ and the left subtree of the root node is $\mathcal{T}$. Note that $\mathcal{T}^{\prime}$ is a width $1$ AL tree with depth $d + 1$. Since $\hat{y}_t = 1$, there exists a hypothesis $h \in V_t$ such that $h(x_t) = 1$. Moreover, for every hypothesis in $h \in V_{t+1} \subset V_t$, we have that $h(x_t) = 0$. Since $\mathcal{T}$ is shattered by  $ V_{t+1} \subset V_t$ and $\mathcal{T}$ is the left subtree of the root node in $\mathcal{T}^{\prime}$, we have that $\mathcal{T}^{\prime}$ is an AL tree of width $1$ and depth $d+1$ shattered by $V_{t}$. However, this contradicts our assumption that $\operatorname{AL}_{1}(V_{t}) = d$. Thus, it must be the case $\operatorname{AL}_{1}(V_{t+1}) \leq \operatorname{AL}_{1}(V_{t}) - 1$ whenever the algorithm errs, and the algorithm can err at most $\operatorname{AL}_1(\mathcal{H})$ times before $\operatorname{AL}_1(V_t) = 0$.
\end{proof}

We extend the results of Theorem \ref{thm:detrealupAT1} to hypothesis classes where $\operatorname{L}(\mathcal{H}) = 1$. Note that $\operatorname{AL}_{1}(\mathcal{H})$ can be much larger than $\operatorname{L}(\mathcal{H})$ even when $\operatorname{L}(\mathcal{H}) = 1$. For example, for the class of singletons $\mathcal{H} = \{x \mapsto \mathbbm{1}\{x = a\}: a \in \mathbbm{N}\}$, we have that $\operatorname{L}(\mathcal{H}) = 1$ but $\operatorname{AL}_1(\mathcal{H}) = \infty$. 

\begin{theorem} [Deterministic realizable upper bound for $\operatorname{L}(\mathcal{H}) = 1$] \label{thm:detrealupL1}
\noindent For any $\mathcal{H} \subseteq \{0, 1\}^{\mathcal{X}}$ such that $\operatorname{L}(\mathcal{H}) = 1$, there exists a deterministic learner which, under apple tasting feedback, makes at most $
1 + 2\sqrt{T}$ mistakes in the realizable setting. 
\end{theorem}

\begin{proof}
    We will show that Algorithm \ref{alg:detL1learner} makes at most $1 + 2\sqrt{T}$ mistakes in the realizable setting under apple tasting feedback after tuning $r$.

    \begin{algorithm}
    \label{alg:detL1learner}
    \setcounter{AlgoLine}{0}
    \caption{Deterministic Realizable Algorithm For Apple Tasting}\label{alg:detL1learner}
    \KwIn{ $V_1 = \mathcal{H}$ and $r > 0$}

    \textbf{Initialize:} $C(h) = 0$ for all $h \in \mathcal{H}$
    
    \For{$t = 1,...,T$} {
    Receive example $x_t$
    
    For  each $y \in \{0,1\} $, define $V_{t}^{y} = \{h \in V_{t} \, \mid h(x_t) = y\}$. 

    \uIf{$V_t(x_t) = \{y\}$} {
        Predict $\hat{y}_t = y$
    }

    \uElseIf{$\operatorname{L}(V_t^0) = 0$} {
    
    Predict $\hat{y}_t = 1$
    
    Observe true label $y_t$
    
    Update $V_{t+1} = V_t^{y_t}$
    
    }\uElseIf{$\exists h \in V_{t}^1$ such that $C(h) \geq r$} {
    
        Predict $\hat{y}_t = 1$
        
        Observe true label $y_t$
        
        Update $V_{t+1} = V_t^{y_t}$ 
    } \uElse{
        Predict $\hat{y}_t = 0$

        \For{$h \in V_t^1$} {
           {                
                Update $C(h) \mathrel{{+}{=}} 1$
            }
        }
        
        Set $V_{t+1} = V_t$
            }
        }
            
    \end{algorithm}

    Let $S = (x_1, h^{\star}(x_t)), ..., (x_T, h^{\star}(x_T))$ be the stream observed by the learner, where $h^{\star} \in \mathcal{H}$ is the optimal hypothesis. As in the proof of Lemma \ref{lem:conver}, consider splitting the stream into the following three parts. Let $S_1$ denote those rounds where $\operatorname{L}(V_t^0) = 0$ but $y_t = 0$. Let $S_2$ denote the rounds where $\operatorname{L}(V_t^0) = 1$, $\hat{y}_t = 1$, but $y_t = 0$. Finally, let $S_3$ denote the rounds where $\operatorname{L}(V_t^0) = 1$, $\hat{y}_t = 0$, but $y_t = 1$. The  number of mistakes Algorithm \ref{alg:detL1learner} makes on the stream $S$ is at most $|S_1| + |S_2| + |S_3|$. We now upper bound each of these terms separately. 

    Starting with $S_1$, observe that if $\operatorname{L}(V_t^0) = 0$, then $|V_t^0| \leq 1$. Thus, if $y_t = 0$, Algorithm \ref{alg:detL1learner} correctly identifies the hypothesis labeling the data stream and does not make any further mistakes. Accordingly, we have that $|S_1| \leq 1$. 

    Next, $|S_2|$ is at most the number of times that Algorithm \ref{alg:detL1learner} predicts $1$ when $\operatorname{L}(V_t^0) = 1$.  Note that if $\operatorname{L}(V_t^0) = 1$ then $|V_t^1| \leq 1$. Thus, by the end of the game, there can be at most $\frac{|\{t: \operatorname{L}(V_t^0) = 1\}|}{r}$ hypothesis $h \in \mathcal{H}$ such that $C(h) \geq r$.  Since Algorithm \ref{alg:detL1learner} only predicts $1$ when there exists a hypothesis in $V_t^1$ with count at least $r$, we have that $|S_2| \leq \frac{|\{t: \operatorname{L}(V_t^0) = 1\}|}{r} \leq \frac{T}{r}$. 

    Finally, we claim that $|S_3| \leq r$. Suppose for the sake of contradiction that $|S_3| \geq r + 1$. Then, by definition, there exists $r+1$ rounds where $\operatorname{L}(V_t^0) = 1$, $\hat{y}_t = 0$ but $y_t = 1$. However, if $\operatorname{L}(V_t^0) = 1$ and $y_t = 1$, then $V_t^1 = \{h^{\star}\}$. Therefore, on the $r+1$'th round where $\operatorname{L}(V_t^0) = 1$, $\hat{y}_t = 0$, and $y_t = 1$, it must be the case $C(h^{\star}) \geq r$. However, if this were true, then the Algorithm would have predicted $\hat{y}_t = 1$ on the $r+1$'th round, a contradiction. Thus, it must be the case that $|S_3| \leq r$.  

    Putting it all together, Algorithm \ref{alg:detL1learner} makes at most $1 + \frac{T}{r} + r$ mistakes. Picking $r = \sqrt{T}$, gives the mistake bound $1 + 2\sqrt{T}$, completing the proof. 
\end{proof}

We highlight that Theorem \ref{thm:detrealupL1} is tight up to constants factors. Indeed, for the class $\mathcal{H}$ of singletons over $\mathbbm{N}$, we have that $\operatorname{W}(\mathcal{H}) = 2$. Therefore, Theorem \ref{thm:real} implies the lower bound of $\frac{\sqrt{T}}{8}$. 

\section{Proof of Lemma \ref{lem:struct}} \label{app:struct} 
To see (i), observe that given any shattered AL tree $\mathcal{T}$ of depth $d$ and width $w_2 > w_1$, we can truncate paths with more than $w_1$ ones to get a shattered AL tree $\mathcal{T}^{\prime}$ of the same depth where now every path has at most $w_1$ ones and the right most path has exactly $w_1$ ones.

To see (ii), consider the case where $w \leq \operatorname{L}(\mathcal{H})$. Then, by property (i), we have that $\operatorname{AL}_{w}(\Hcal) \geq \operatorname{AL}_{\operatorname{L}(\mathcal{H})}(\Hcal) \geq \operatorname{L}(\mathcal{H}) \geq w.$ If $w > \operatorname{L}(\mathcal{H})$, then $\operatorname{AL}_{w}(\Hcal) \geq \operatorname{L}(\mathcal{H})$ which follows from the fact that an AL tree $\mathcal{T}$ of width $w$ and depth $\operatorname{L}(\mathcal{H}) < w$ is a complete binary tree of depth $\operatorname{L}(\mathcal{H})$.

To see (iii), fix $w \geq \operatorname{L}(\Hcal) + 1$. Then, by property (ii), we have that $\operatorname{AL}_{w}(\Hcal) \geq \operatorname{L}(\mathcal{H})$. Thus, it suffices to show that $\operatorname{AL}_w(\mathcal{H}) \leq \operatorname{L}(\Hcal).$ Suppose for the sake of contradiction that $\operatorname{AL}_w(\mathcal{H}) \geq \operatorname{L}(\Hcal) + 1.$ Then, using property (i) and the fact that $w \geq \operatorname{L}(\Hcal) + 1$, we have that $\operatorname{AL}_{\operatorname{L}(\Hcal) + 1}(\Hcal) \geq \operatorname{AL}_{w}(\Hcal) \geq \operatorname{L}(\Hcal) + 1$. Thus, by definition of ALdim, there exists a Littlestone tree of depth $\operatorname{L}(\Hcal) + 1$ shattered by $\mathcal{H}$, a contradiction. 

To see (iv), note that when $\operatorname{L}(\Hcal) < \infty$, we have that $\operatorname{AL}_{\operatorname{L}(\Hcal) + 1}(\Hcal) = \operatorname{L}(\mathcal{H})$ by property (iii). Thus, by definition of the Effective width, it must be the case that $\operatorname{W}(\mathcal{H}) \leq \operatorname{L}(\mathcal{H}) + 1$.

To see (v), it suffices to prove that $\operatorname{L}(\mathcal{H}) = \infty \implies \operatorname{W}(\mathcal{H}) = \infty$ since (iv) shows that $\operatorname{L}(\Hcal) < \infty \implies \operatorname{W}(\Hcal) < 
\infty$. This is true because if $\operatorname{L}(\mathcal{H}) = \infty$, then for any width $w \in \mathbbm{N}$ and depth $d \in \mathbbm{N}$, one can always prune a shattered Littlestone tree of depth $d$ to get a shattered AL tree of depth $d$ and width $w$. 

\section{Proof of Lemma \ref{lem:conver}} \label{app:conv}
\begin{algorithm}
\setcounter{AlgoLine}{0}
\caption{Conversion of Full-Information Algorithm to Apple Tasting Algorithm }\label{alg:conv}
\KwIn{Full-Information Algorithm $\mathcal{A}$, 
 false negative mistake bound $M_{-}$ of $\mathcal{A}$}

\For{$t = 1,...,T$} {
    Receive $x_t$ and query $\mathcal{A}$ to get $\xi_t = \mathcal{A}(x_t)$.

    Draw $r \sim \text{Unif}([0,1])$ and predict 
    \[\hat{y}_t = \begin{cases}
        &1 \quad \quad \text{if } \xi_t  =1.   \\
        & 1 \quad \quad \text{if } \xi_t  = 0 \text{ and } r \leq \sqrt{M_{-}/T}.\\
        &0 \quad \quad \text{otherwise}.
    \end{cases}\]

    If $\hat{y}_t = 1$, receive $y_t$ and update $\mathcal{A}$ by passing $(x_t, y_t)$.

}
    
\end{algorithm}
If $T \leq M_{-}$, then the claimed expected mistake bound is $\geq T$, which trivially holds for any algorithm. So, we only consider the case when $T > M_{-}$. Let $\mathcal{A}$ be a deterministic online learner, which makes at most $M_{-}$ false negative mistakes and at most $M_{+}$ false positive mistakes under full-information feedback. We now show that Algorithm \ref{alg:conv}, a randomized algorithm that uses $\mathcal{A}$ in a black-box fashion, has expected mistake bound at most $M_{+} \, + \, 2\sqrt{T M_{-}}$ in the realizable setting under apple-tasting feedback.

For each bitstring $b \in \{0,1\}^3$, define $S_b = \{ t\in [T] \, \mid b_1 = \xi_t, b_2 = \hat{y}_t, \text{and } b_3 = y_t \}.$ Here, $b_1,b_2, b_3$ are the first, second, and third bits of the bitstring $b$.  Using this notation, we can write the expected mistake bound of Algorithm \ref{alg:conv} as
\[\mathbb{E} \left[ \sum_{t=1}^T \indicator\{\hat{y}_t \neq y_t\}\right] = \mathbb{E} \left[\,  |S_{101}|+  |S_{001}|  + |S_{110}|  + |S_{010}|  \,\right].\]

Since $\hat{y}_t =1$ whenever $\xi_t=1$, we have $|S_{101}| =0$. 
Note that $|S_{001}| \leq N$, where $N$ is the number of failures before $M_{-}$ successes in independent Bernoulli trials with probability $\sqrt{M_{-}/T}$ of success. That is, $N$ quantifies the number of rounds before $\xi_t$ is flipped $M_{-}$ number of times from $0$ to $1$  in rounds when $y_t=1$. Recalling that $N \sim \text{Negative-Binomial}(M_{-}, \sqrt{M_{-}/T})$, we have
\[\mathbb{E}[|S_{001}|] \leq\mathbb{E}[N]  \leq M_{-}\left( \sqrt{\frac{T}{M_{-}}}-1 \right) \leq \sqrt{M_{-}T} - M_{-}. \]
 Moreover, using the fact that $\mathcal{A}$ makes at most $M_{+}$ false positive mistakes, we have $|S_{110}| \leq M_{+}$. 
 
 Finally, using the prediction rule in Algorithm \ref{alg:conv}, we have 
\[\mathbb{E}[|S_{010}|] \leq \mathbb{E} \left[ \sum_{t=1}^{T} \indicator\{\xi_t =0 \,\, \text{ and } \,\, \hat{y}_t =1\}\right] \leq \mathbb{E} \left[ \sum_{t=1}^{T} \indicator\left\{r \leq \sqrt{\frac{M_{-}}{T}} \right\}\right] \leq T \, \sqrt{\frac{M_{-}}{T}} = \sqrt{M_{-} T}. \]
Putting everything together, we have 
\[\mathbb{E} \left[ \sum_{t=1}^T \indicator\{\hat{y}_t \neq y_t\}\right] \leq \sqrt{M_{-}T} - M_{-} + M_{+} + \sqrt{M_{-}T} \leq M_{+} + 2\sqrt{M_{-}T}. \]
 This completes our proof.

\section{Proof of Lemma \ref{lem:exp4rb}} \label{app:experts}
Observe that $\hat{\ell}_t(y) \leq \frac{1}{\eta}$ for $y \in \{0, 1\}$ since $p_t^1 \geq \eta$. Let $\bar{\ell}_t = \sum_{y \in \{0, 1\}} p_t^y\hat{\ell}_t(y)$ and define $\ell^{\prime}_t$ such that $\ell^{\prime}_t(y) = \hat{\ell}_t(y) - \bar{\ell}_t$ for all $y \in \{0, 1\}$. Notice that executing EXP4.AT on the loss vectors $\hat{\ell}_1, ..., \hat{\ell}_T$ is equivalent to executing EXP4.AT on the loss vectors $\ell^{\prime}_1, ..., \ell^{\prime}_T$. Indeed, since $\bar{\ell}_t$ is constant over the experts, the weights $q_{t}^i$ remained unchanged regardless of whether $\ell^{\prime}_t$ or $\hat{\ell}_t$ is used to update the experts. Moreover, we have that $\ell^{\prime}_t(y) \geq -\frac{1}{\eta}$.

 We start by following the standard analysis of exponential weighting schemes.  Let $w_1^i = 1$, $w_{t+1}^i = w_t^i \text{exp}(-\eta \ell^{\prime}_t(\mathcal{E}_t^i))$, and $W_t = \sum_{i=1}^N w_t^i$.  Then, $q_t^i = \frac{w_t^i}{W_t}$ and we have
\begin{align*}
 \frac{W_{t+1}}{W_t} &= \sum_{i=1}^N \frac{w_{t+1}^i}{W_t} \\
        &= \sum_{i=1}^N \frac{w_{t}^i\text{exp}(-\eta \ell^{\prime}_t(\mathcal{E}_t^i))}{W_t} \\
        &= \sum_{i=1}^N q_t^i\text{exp}(-\eta \ell^{\prime}_t(\mathcal{E}_t^i)) \\
        &\leq \sum_{i=1}^N q_t^i\left(1 -\eta (\ell^{\prime}_t(\mathcal{E}_t^i)) + \eta^2 (\ell^{\prime}_t(\mathcal{E}_t^i))^2\right) \\
        &= 1 - \eta \sum_{i=1}^N q_t^i \ell^{\prime}_t(\mathcal{E}_t^i) + \eta^2 \sum_{i=1}^N q_t^i (\ell^{\prime}_t(\mathcal{E}_t^i))^2, 
\end{align*}

where the inequality follows from the fact that $\ell^{\prime}_t(\mathcal{E}_t^i) \geq -\frac{1}{\eta}$ and $e^x \leq 1 + x + x^2$ for all $x \leq 1$. Taking logarithms, summing over $t$, and using the fact that $\ln(1-x) \leq -x$ for all $x \geq 0$ we get

$$\ln\frac{W_{T+1}}{W_1} \leq - \eta \sum_{t=1}^T \sum_{i=1}^N q_t^i \ell^{\prime}_t(\mathcal{E}_t^i) + \eta^2 \sum_{t=1}^{T} \sum_{i=1}^N q_t^i (\ell^{\prime}_t(\mathcal{E}_t^i))^2.$$

Also, for any expert $j \in [N]$, we have

$$\ln\frac{W_{T+1}}{W_1} \geq \ln \frac{w_{T+1}^j}{W_1} = -\eta \sum_{t=1}^T \ell^{\prime}_t(\mathcal{E}_t^j)) - \ln N. $$

Combining this with the upper bound on $\ln\frac{W_{T+1}}{W_1}$, rearranging, and dividing by $\eta$, we get

$$\sum_{t=1}^T \sum_{i=1}^N q_t^i \ell^{\prime}_t(\mathcal{E}_t^i) \leq \sum_{t=1}^T \ell^{\prime}_t(\mathcal{E}_t^j) + \frac{\ln N}{\eta} + \eta \sum_{t=1}^{T} \sum_{i=1}^N q_t^i (\ell^{\prime}_t(\mathcal{E}_t^i))^2.$$


Using the definition of $\ell^{\prime}_t$, we further have that 

$$\sum_{t=1}^T \sum_{i=1}^N q_t^i \hat{\ell}_t(\mathcal{E}_t^i) \leq \sum_{t=1}^T \hat{\ell}_t(\mathcal{E}_t^j) + \frac{\ln N}{\eta} + \eta \sum_{t=1}^{T} \sum_{i=1}^N q_t^i (\ell^{\prime}_t(\mathcal{E}_t^i))^2.$$

Next, observe that 
$$ \sum_{i=1}^N q_t^i \hat{\ell}_t(\mathcal{E}_t^i)  = \left(\sum_{i=1}^N q_t^i \mathcal{E}_t^i\right) \hat{\ell}_t(1) + \left(1 - \sum_{i=1}^N q_t^i \mathcal{E}_t^i\right) \hat{\ell}_t(0) = \frac{1}{1-\eta}\sum_{y \in \{0, 1\}} p_t^y \hat{\ell}_t(y) - \frac{\eta}{1-\eta} \hat{\ell}_t(1).$$

Moreover,
\begin{align*}
 \sum_{i=1}^N q_t^i (\ell^{\prime}_t(\mathcal{E}_t^i))^2 &= \sum_{i=1}^N q_t^i \left(\sum_{y \in \{0, 1\}} \mathbbm{1}\{y = \mathcal{E}_t^i\} \ell^{\prime}_t(y) \right)^2 \\
        & = \sum_{i=1}^N q_t^i \left(\sum_{y \in \{0, 1\}}\mathbbm{1}\{y = \mathcal{E}_t^i\} \ell^{\prime}_t(y)^2 \right)\\
        &= \sum_{y \in \{0, 1\}} \left(\sum_{i=1}^N q_t^i \mathbbm{1}\{y = \mathcal{E}_t^i\}\right) \ell^{\prime}_t(y)^2 \\
         &= \left(\sum_{i=1}^N q_t^i \mathcal{E}_t^i\right) \ell^{\prime}_t(1)^2 + \left(1 - \sum_{i=1}^N q_t^i \mathcal{E}_t^i\right) \ell^{\prime}_t(0)^2 \\
         &\leq \frac{1}{1-\eta}\sum_{y \in \{0, 1\}} p_t^y \ell^{\prime}_t(y)^2.
\end{align*}
Therefore, for any fixed expert $j$, 
    $$\frac{1}{1-\eta} \sum_{t=1}^T \sum_{y \in \{0, 1\}} p_t^y \hat{\ell}_t(y) - \frac{\eta}{(1 - \eta)} \sum_{t=1}^T \hat{\ell}_t(1) \leq  \sum_{t=1}^T \hat{\ell}_t(\mathcal{E}_t^{j}) + \frac{\ln N}{\eta} + \frac{\eta}{1-\eta}\sum_{t=1}^T \sum_{y \in \{0, 1\}} p_t^y \ell^{\prime}_t(y)^2.$$ 

Multiplying by $1 - \eta$ and rearranging, we have 

$$\sum_{t=1}^T \sum_{y \in \{0, 1\}} p_t^y \hat{\ell}_t(y) - (1-\eta) \sum_{t=1}^T \hat{\ell}_t(\mathcal{E}_t^{j}) \leq \frac{(1 - \eta)\ln N}{\eta} + \eta \sum_{t=1}^T \hat{\ell}_t(1) + \eta\sum_{t=1}^T \sum_{y \in \{0, 1\}} p_t^y \ell^{\prime}_t(y)^2$$ which further implies the guarantee:

\begin{align*}
\sum_{t=1}^T \sum_{y \in \{0, 1\}} p_t^y \hat{\ell}_t(y) - \sum_{t=1}^T  \hat{\ell}_t(\mathcal{E}_t^{j}) &\leq \frac{\ln N}{\eta} + \eta \sum_{t=1}^T \hat{\ell}_t(1) + \eta\sum_{t=1}^T \sum_{y \in \{0, 1\}} p_t^y \ell^{\prime}_t(y)^2 \\
&=  \frac{\ln N}{\eta} + \eta \sum_{t=1}^T \hat{\ell}_t(1) + \eta\sum_{t=1}^T \sum_{y \in \{0, 1\}} p_t^y (\hat{\ell}_t(y) - \bar{\ell}_t)^2.
\end{align*}

Next, note that 

\begin{align*}
 \sum_{y \in \{0, 1\}} p_t^y (\hat{\ell}_t(y)-\bar{\ell}_t)^2 &= \sum_{y \in \{0, 1\}} p_t^y \hat{\ell}_t(y)^2 -\left(\sum_{y \in \{0, 1\}} p_t^y\hat{\ell}_t(y)\right)^2 \\
        &\leq \sum_{y \in \{0, 1\}} p_t^y \hat{\ell}_t(y)^2 -\sum_{y \in \{0, 1\}} (p_t^y)^2 \hat{\ell}_t(y)^2 \\
        &= \sum_{y \in \{0, 1\}} p_t^y (1 - p_t^y) \hat{\ell}_t(y)^2 \\
        &\leq p_t^1(1 - p_t^1) \hat{\ell}_t(0)^2 + p_t^1 \hat{\ell}_t(1)^2,
\end{align*}

where the first inequality is true because of the nonnegativity of the losses $\hat{\ell}_t$ and the last inequality is true because $0 \leq p_t^1 \leq 1$. Putting things together, we have that 

$$ \sum_{t=1}^T \sum_{y \in \{0, 1\}} p_t^y \hat{\ell}_t(y) - \sum_{t=1}^T \hat{\ell}_t(\mathcal{E}_t^{j}) \leq \frac{\ln N}{\eta} + \eta \sum_{t=1}^T \hat{\ell}_t(1) + \eta\sum_{t=1}^T p_t^1 (1 - p_t^1) \hat{\ell}_t(0)^2 + \eta\sum_{t=1}^T p_t^1 \hat{\ell}_t(1)^2.$$

\noindent Since expert $j \in [N]$ was arbitrary, this completes the proof.

\section{Proof of Theorem \ref{thm:exp4rb}} \label{app:thmexp}
    From Lemma \ref{lem:exp4rb}, we have that for an fixed expert $j \in [N]$

    $$\sum_{t=1}^T \sum_{y \in \{0, 1\}} p_t^y \hat{\ell}_t(y) - \sum_{t=1}^T \hat{\ell}_t(\mathcal{E}_t^{j}) \leq \frac{\ln N}{\eta} + \eta \sum_{t=1}^T \hat{\ell}_t(1) + \eta\sum_{t=1}^T p_t^0 p_t^1 \hat{\ell}_t(0)^2 + \eta\sum_{t=1}^T p_t^1 \hat{\ell}_t(1)^2.$$

    Taking expectations on both sides and using the fact that $\mathbb{E}_{t}\left[\hat{\ell}_t(y) \right] = \mathbbm{1}\{y \neq y_t\}$,   $\mathbb{E}_t\left[\hat{\ell}_t(y)^2 \right] = \frac{\mathbbm{1}\{y \neq y_t\}}{p_t^1}$ gives
\begin{align*}
    \mathbb{E}\left[\sum_{t=1}^T \mathbbm{1}\{\hat{y}_t \neq y_t\} \right] -  \sum_{t=1}^T \mathbbm{1}\{\mathcal{E}_t^j \neq y_t\} &\leq \frac{\ln N}{\eta} + \eta \sum_{t=1}^T \mathbbm{1}\{1 \neq y_t\} + \eta \sum_{t=1}^T \mathbb{E}\left[p_t^0 p_t^1 \frac{\mathbbm{1}\{0 \neq y_t\}}{p_t^1} + p_t^1 \frac{\mathbbm{1}\{1 \neq y_t\}}{p_t^1}\right] \\
    &\leq \frac{\ln N}{\eta} + 2\eta T.
\end{align*}

    Substituting $\eta =  \sqrt{\frac{\ln N}{2T}}$, we have 

    $$\mathbb{E}\left[\sum_{t=1}^T \mathbbm{1}\{\hat{y}_t \neq y_t\} \right] -  \sum_{t=1}^T \mathbbm{1}\{\mathcal{E}_t^j \neq y_t\} \leq 2 \sqrt{2T\ln N} \leq 3\sqrt{T \ln N},$$

    which completes the proof. 

\section{Proof of Theorem \ref{thm:agn}} \label{app:agn}

Let $(x_1, y_1), ..., (x_T, y_T)$ denote the stream of labeled instances to be observed by the agnostic learner and let $h^{\star} = \argmin_{h \in \mathcal{H}} \sum_{t=1}^T \mathbbm{1}\{h(x_t) \neq y_t\}$ be the optimal hypothesis in hindsight. Given the time horizon $T$, let $L_T = \{L \subset [T]: |L| \leq \operatorname{L}(\mathcal{H})\}$ denote the set of all possible subsets of $[T]$ of size of $\operatorname{L}(\mathcal{H})$. For every $L \in L_T$, define an expert $\mathcal{E}_L$, whose prediction on time point $t \in [T]$ on instance $x_t$ is defined by

$$\mathcal{E}_L(x_t) = \begin{cases}
			\text{SOA}\left(x_t| \{(x_i, \mathcal{E}_L(x_i))\}_{i =1}^{t-1} \right), & \text{if $t \notin L$}\\
            \neg \text{SOA}\left(x_t| \{(x_i, \mathcal{E}_L(x_i))\}_{i =1}^{t-1} \right), & \text{otherwise}
		 \end{cases}$$

where $\text{SOA}\left(x_t| \{(x_i, \mathcal{E}_L(x_i))\}_{i =1}^{t-1} \right)$ denotes the prediction of the SOA on the instance $x_t$  after running and updating on the labeled stream $\{(x_i, \mathcal{E}_L(x_i))\}_{i =1}^{t-1}$. Let $E = \{\mathcal{E}_L: L \in L_T\}$ denote the set of all Experts parameterized by subsets $L \in L_T$. Observe that $|E| \leq T^{\operatorname{L}(\mathcal{H})}$. 

We claim that there exists an expert $\mathcal{E}_{L^{\star}} \in E$ such that for all $t \in [T]$, we have that $\mathcal{E}_{L^{\star}}(x_t) = h^{\star}(x_t)$. To see this, consider the hypothetical stream of instances labeled by the optimal hypothesis $S^{\star} = \{(x_t, h^{\star}(x_t))\}_{t=1}^T$. Let $L^{\star} = \{t_1, t_2, ... \}$ be the indices on which the SOA would have made mistakes had it run and updated on $S^*$. By the guarantee of SOA, we know that $|L^{\star}| \leq \operatorname{L}(\mathcal{H})$. By construction of $E$, there exists an expert $\mathcal{E}_{L^{\star}}$ parameterized by $L^{\star}$. We claim that for all $t \in [T]$, we have that $\mathcal{E}_{L^{\star}}(x_t) = h^{\star}(x_t)$. This follows by strong induction on $t \in [T]$. For the base case $t = 1$, there are two subcases to consider. If $1 \in  L^{\star}$, then we have that $\mathcal{E}_{L^{\star}}(x_1) = \neg \text{SOA}\left(x_1| \{\} \right) = h^{\star}(x_1)$, by definition of $L^{\star}$. If $1 \notin L^{\star}$, then $\mathcal{E}_{L^{\star}}(x_1) = \text{SOA}\left(x_1| \{\} \right) = h^{\star}(x_1)$ also by definition of $L^{\star}$. Now for the induction step, suppose that $\mathcal{E}_{L^{\star}}(x_i) = h^{\star}(x_i)$ for all $i \leq t$. Then, if $t+1 \in L^{\star}$, we have that $\mathcal{E}_{L^{\star}}(x_{t+1}) = \neg \text{SOA}\left(x_{t+1}| \{(x_i, \mathcal{E}_{L^{\star}}(x_i))\}_{i=1}^{t} \right) = \neg \text{SOA}\left(x_{t+1}| \{(x_i, h^{\star}(x_i))\}_{i=1}^{t} \right) = h^{\star}(x_{t+1}).$ If $t+1 \notin L^{\star}$, then $\mathcal{E}_{L^{\star}}(x_{t+1}) = \text{SOA}\left(x_{t+1}| \{(x_i, \mathcal{E}_{L^{\star}}(x_i))\}_{i=1}^{t} \right) = \text{SOA}\left(x_{t+1}| \{(x_i, h^{\star}(x_i))\}_{i=1}^{t} \right) = h^{\star}(x_{t+1}).$  The final equality in both cases are due to the definition of $L^{\star}$.  

Now, consider the agnostic online learner $\mathcal{A}$ that runs EXP4.AT using $E$. By Theorem \ref{thm:exp4rb}, we have that

\begin{align*}
\mathbbm{E}\left[\sum_{t=1}^T \mathbbm{1}\{\mathcal{A}(x_t) \neq y_t\} \right] &\leq \inf_{\mathcal{E} \in E} \sum_{t=1}^T \mathbbm{1}\{\mathcal{E}(x_t) \neq y_t\} + 3\sqrt{T\ln{|E|}}\\
&\leq \sum_{t=1}^T \mathbbm{1}\{\mathcal{E}_{L^{\star}}(x_t) \neq y_t\} + 3\sqrt{\operatorname{L}(\mathcal{H})T\ln{T}}\\
&= \sum_{t=1}^T \mathbbm{1}\{h^{\star}(x_t) \neq y_t\} + 3\sqrt{\operatorname{L}(\mathcal{H})T\ln{T}}.
\end{align*}

Thus, $\mathcal{A}$ achieves the stated upper bound on expected regret under apple tasting feedback, which completes the proof.  

\section{Effective width of the $k$-wise generalization of $\Hcal_{\text{sing}}$} \label{app:examples}


In this section, we compute the Effective width of the $k$-wise generalization of the class of singletons $\Hcal_{\text{sing}}$. 

\begin{proposition} Let $\Xcal = \mathbbm{N}$ and $\Hcal_k = \{x \mapsto \mathbbm{1}\{x \in A\} : A \subset \mathbbm{N}, |A| \leq k\}$. Then, $\operatorname{W}(\Hcal) = k+1$ and $\operatorname{AL}_{\operatorname{W}(\Hcal)} = 0$.
\end{proposition}

\begin{proof} Consider an Apple Littlestone tree $\Tcal(w, d)$ of width $w = k$ and depth $d \geq w$ such that all the internal nodes on level $i \in [d]$ are labeled by the instance $i \in \mathbbm{N}$. It is not too hard to see that $\Hcal_k$ shatters $\Tcal(w, d)$. Since $d \geq w$ was chosen arbitrarily, this is true for all $d \in \mathbbm{N}$ and thus $\operatorname{AL}_k(\Hcal_k) = \infty$. On the other hand, consider an Apple Littlestone tree $\Tcal^{\prime}(w^{\prime}, d)$ of width $w^{\prime} = k+1$ and depth $d \in \mathbbm{N}$. Note that in order to shatter $\Tcal^{\prime}$, there must exist a hypothesis that outputs at least $k+1$ ones across $k+1$ distinct instances in $\Xcal$. However, by definition, every hypothesis $h \in \Hcal$ outputs $1$ on at most $k$ distinct instances. Thus, $\Tcal^{\prime}$ cannot be shattered by $\Hcal_{k+1}$. Since this is true for all such $d \in \mathbbm{N}$, we have that $\operatorname{AL}_{k+1}(\Hcal_k) = 0$. This completes the proof as it must be the case that $\operatorname{W}(\Hcal) = k+1.$
\end{proof}


\end{document}